\documentclass{article}
\usepackage[margin=1in]{geometry}
\usepackage[utf8]{inputenc}
\usepackage{graphicx}
\usepackage{caption}
\usepackage{amsmath}
\usepackage{amssymb}
\usepackage{float}
\usepackage{hyperref}
\usepackage{amsthm}
\usepackage{xcolor}
\usepackage{comment}
\usepackage{amsfonts}
\usepackage{dsfont}
\usepackage{stmaryrd}
\usepackage{xcolor}
\usepackage{subcaption}

\DeclareMathOperator{\R}{\mathbb{R}}

\def\Dc{{\cal D}}
\def\Lc{{\cal L}}

\def\Kc{{\cal K}}

\def\Lc{{\cal L}}

\def\Nc{{\cal N}}
\def\Pc{{\cal P}}

\def\Vc{{\cal V}}

\def \Yc{{\cal Y}}
\def \Wc{{\cal W}}

\def \E{\mathbb{E}}

\def \1{{\bf 1}}

\def \N{\mathbb{N}}

\def \b1{\bf{1}}

\def\boQ{{\boldsymbol Q}}
\def\boM{{\boldsymbol M}}
\def\boL{{\boldsymbol L}}
\def\boS{{\boldsymbol S}}
\def\boV{{\boldsymbol V}}

\def \d{\mathrm{d}} 
 
\def \mrp{\mathrm{p}} 
\def \boQ{\mathrm{Q}}

\def\blue[#1]{{\color{blue}#1}}

\def\beqs{\begin{eqnarray*}}
\def\enqs{\end{eqnarray*}}
\def\beq{\begin{eqnarray}}
\def\enq{\end{eqnarray}}

\def\argmin{\mathop{\rm argmin}}
\def\argmin_#1{\underset{#1}{\mathrm{argmin\, }}}

\newtheorem{Theorem}{Theorem}[section]

\newtheorem{Remark}[Theorem]{Remark}

\usepackage{mathtools}
\mathtoolsset{showonlyrefs}

\usepackage[backend=bibtex, maxbibnames=5, style=numeric]{biblatex}

\numberwithin{equation}{section}

\title{Quantile and moment neural networks for learning functionals of distributions.
\thanks{This work is supported by  FiME, Laboratoire de Finance des March\'es de l'Energie, and the ''Finance and Sustainable Development'' EDF - CACIB Chair.}
}

\author{
Xavier \sc{Warin}
\footnote{EDF R\&D \& FiME \sf \href{mailto:xavier.warin at edf.fr}{xavier.warin at edf.fr}}}

\begin{document}

\maketitle
\begin{abstract}
We study  news neural networks to approximate function of distributions in a probability space. 
Two classes of neural networks based on quantile and moment approximation are proposed to learn these functions and are theoretically supported by universal approximation theorems. 
By mixing the quantile and moment features in other new networks, we develop schemes that outperform existing networks on numerical test cases involving univariate distributions. For bivariate distributions, the moment neural network outperforms all other networks.
 \end{abstract}

\section{Introduction}
The deep neural networks have been successfully used to solve high dimensional PDEs  either by solving the PDE using physics informed methods, or by using backward stochastic differential equations (see \cite{beck2020overview}, \cite{GPW21} for an overview).
Recently the  mean field game and control theory has allowed the formalization of problems involving large populations of interacting agents. The solution of such  problems is a function depending on the probability distribution of the population  and can be obtained by solving a PDE in the Wasserstein space of probability measures (called the Master equation) or by solving BSDEs of McKean-Vlasov (MKV)  (see \cite{cardel19,cardel2} ).
In this case, the resulting PDE is infinite dimensional and must be reduced to a  (high) finite dimensional problem to be tractable.\\
To solve such problems, \cite{pham2022mean} has developed two schemes  approximating functions depending on both  $X$ a random variable and $\mu$  a probability distribution, where $X \sim \mu$. 
The first scheme is based on a bin representation of the density  and the second uses a neural network to  automatically extract the key features of the distribution. 
In both cases, these key features and the $X$ values are used as input to a neural network permitting  to approximate the value function.  The first approach is the bin network and the second one is the cylinder network.
Both schemes have been successfully applied to  various toy cases in the case of  one-dimensional distributions and used to solve the master equation using its semi-lagrangian representation in  \cite{warinpham2022meanW}.
As we explain in the next section, the $X$ dependence on the functional is not relevant for testing the various networks developed, and we focus in this article only on the dependence on the distribution.\\
In this article we propose new  different networks to approximate functions depending on distributions:
\begin{itemize}
\item The first one, limited to one-dimensional distributions, uses the quantile of the distribution as key features : the resulting scheme gives the quantile network scheme.
\item The second ones uses the moments of the distribution as key features and leads to the moment network scheme.
\item Finally, the two previous features can be mixed to  take  advantage of the first two  networks.
\end{itemize}
We give some universal approximation theorems for the first two networks.
We test the developed networks on functions of univariate and bivariate distributions,   where possible. We compare the proposed networks with the bin network and the cylinder network and show that:
\begin{itemize}
\item The moment network {\bf or}  the quantile network outperform the cylinder network and the bin network in the case of univariate distributions: the best solution obtained  by the two networks always gives better results  (on our tests) than the cylinder network and the bin network.
\item By combining quantile and moment, we obtain a neural network that always outperforms the cylinder and the bin  networks.
\item In the case of bivariate distributions, the bin network fails and the moment network outperforms all other networks.
\end{itemize}
The structure of the article is as follows. In a first section, we formalize our problem as a minimization problem on the distribution space using a formal neural network. We give the general methodology for sampling distributions in the general multivariate case, and  show how to solve the previous minimization problem  using a stochastic gradient method. 
The second section is dedicated to the  different  proposed neural networks. 
The last one is dedicated to numerical results for univariate and bivariate distributions.
A final conclusion is  given.
\vspace{2mm}

\noindent {\bf Notations.} Denote by $\Pc_2(\R^d)$   the Wasserstein space of square integrable probability measures equipped with the $2$-Wasserstein distance $\Wc_2$. 
%Given $\mu$ $\in$ $\Pc_2(\R^d)$, we denote by $L^2(\mu)$ as the set of measurable functions $\phi$ on $\R^d$ s.t. 
% \beqs
%  |\phi|_\mu^2 &: =&  \; \int |\phi(x)|^2 \mu(\d x) \; < \; \infty. 
% \enqs  
 Given some  $\mu$ $\in$ $\Pc_2(\R^d)$, and $\phi$ a measurable function on $\R^d$ with quadratic growth condition, hence in $L^2(\mu)$, we set: $\E_{X\sim\mu}[\phi(X)]$ $:=$ $\int \phi(x) \mu(\d x)$.

\section{Learning distribution functions}
\label{sec:learning}
Given a function $V$ on $\Pc_2(\R^d)$, valued on $\R^p$, we want to approximate the infinite-dimensional mapping
\begin{align} \label{defVc} 
\Vc : \mu \in \Pc_2(\R^d) & \longmapsto \;  V(\mu) \in  \R^p, 
\end{align} 
called the distribution function, by a map  $\Nc$ constructed from suitable classes of neural networks. The distribution network $\Nc$ takes input $\mu$ a probability measure  and outputs $\Nc(\mu)$. 
The quality of this approximation is measured by the error:
\begin{align}
L(\Nc)  :=  \int_{\Pc_2(\R^d)}  \big| \Vc(\mu) - \Nc(\mu) \big|^2  \nu(\d \mu)
\label{eq:prob}
\end{align}
where $\nu$ is a probability measure on $\Pc_2(\R^d)$, called the training measure. The distribution function $\Vc$ is learned by minimizing the loss function over the parameters of the neural network operator $\Nc$ .\\
In the article \cite{pham2022mean},  the authors learn what they call a mean-field function,  a function  $ \hat V$ depending on both $\mu$ and $x$.
The network is a function $ \hat \Nc$ that takes as input $\mu$ a probability measure  and $x$ in the support of $\mu$, and outputs $\hat \Nc(\mu)(x)$.
The  solution is found by minimizing:
\begin{align}
\hat L( \hat \Nc) :=  \int_{\Pc_2(\R^d)}  \E_{X\sim\mu}\big| \hat V(X,\mu) - \hat \Nc(\mu)(X) \big|^2  \nu(\d \mu).
\label{eq:art1}
\end{align}
The resolution of the equation \eqref{eq:art1} is more general than the resolution of the equation \eqref{eq:prob}, but in fact the result is simply obtained by concatenating $x$ and the representation of the distribution $\mu$ as input to the neural network, similarly as suggested in \cite{gerlauphawar21a}. Therefore we focus on the resolution of the problem \eqref{eq:prob}.\\
In the following, we explain   how to sample distributions and  how to generate samples for a given distribution in the multivariate case.
Then we explain the global methodology used to train the neural networks.  This methodology is used for all the networks developed in the next sections. 
\subsection{Sampling distribution on a compact set}
\label{sec:sampling}
To learn  a function of a distribution  $\mu$ with support in  $\Kc = [\underline \Kc_1, \bar \Kc_1] \times \ldots \times [\underline \Kc_d, \bar \Kc_d] \subset \R^d$, we must be   able to "generate" distributions and, having chosen the distribution $\mu$ , to  efficiently sample    $X \sim \mu$ in $\R^d$.\\
As done in \cite{pham2022mean},  we use a bin representation but propose a  different algorithm to tackle the multivariate case.  
By tensorization, a multivariate bin representation for a lattice  $(J_1,\ldots, J_d)$ is given,  for $( j_1, \ldots , j_d) \in [1,J_1] \times \ldots \times [1,J_d]$,  by 
$$ {\rm Bin}(j_1,\ldots,j_d) = \prod_{i=1}^d [\underline \Kc_i + (j_i-1) \frac{\bar \Kc_i - \underline \Kc_i }{J_i}, \underline \Kc_i + j_i \frac{\bar \Kc_i - \underline \Kc_i }{J_i}].$$
\begin{itemize}
\item First, we generate a distribution $\mu$   by sampling $e_1, \ldots , e_{\prod_{i=1}^d J_i}$,  positive random variables  according to an exponential law, and set for $( j_1, \ldots , j_d) \in [1,J_1] \times \ldots \times [1,J_d]$ 
$$p(j_1, \ldots, j_d) = \frac{e_{j_1 +j_2 J_1 +\ldots  j_d \prod_{i=1}^{d-1} J_i}}{\sum_{i=1}^{\prod_{i=1}^d J_i} e_i}$$
which gives a constant per bin probability measure where the  probability of sampling in ${\rm Bin}(j_1,\ldots,j_d)$ is given by  $p(j_1, \ldots, j_d)$.
\item  Now that we have chosen $\mu$,
we can generate $N$ samples of $d$ dimensional coordinates  $(j_1^n, \ldots , j_d^n) \in [1,J_1] \times \ldots \times [1,J_d]$  for $n \in [1,N]$ such that 
$${\rm proba}[(j_1^n, \ldots , j_d^n) = (j_1, \ldots , j_d)] = p(j_1, \ldots, j_d).$$
Finally, we sample $Y^n \sim \bold{U}([0,1]^d) $  for $n=1, N$, and set
\begin{align}
    X^n =& (X_1^n, \ldots, X_d^n), \\
    { \rm where } \quad  & X_i^n  = \underline \Kc_i + (j_i^n-1 + Y_i^n) \frac{\bar \Kc_i - \underline \Kc_i }{J_i},  {\rm for } \quad  i=1, \dots, d.
\end{align}
\end{itemize}
\begin{Remark}
This procedure allows to generate points according to a constant  density function per bin. In dimension one, it is equivalent to the algorithm proposed in \cite{pham2022mean}  which generates points with a linear representation of the cumulative distribution function.
\end{Remark}

\subsection{The training methodolody}
Since the  equation \eqref{eq:prob} is infinite dimensional, we need to introduce a discretization of the measure. We note  $R^K(\mu) := (R_{k}^K(\mu))_{k=1,\ldots,K}$ the $K$ features estimated from the law $\mu$. The features selected depend on the method developed and will be detailed in the following sections.\\
The neural network $ \Nc(\mu) :=  \Phi_\theta(R^K(\mu))$  is such that  $\Phi_\theta$ is  an operator from $\R^K$ to $\R^p$ depending on some parameters $\theta$ and  we use a gradient descent algorithm 
(ADAM  \cite{kingma2014adam}) with Tensorflow software \cite{abadi2016tensorflow} to minimize the loss
\begin{align}
 \bar  L(\theta)  :=  \int_{\Pc_2(\R^d)}  \big| \Vc(\mu) - \Phi_\theta(R^K(\mu)) \big|^2  \nu(\d \mu)
\label{eq:probTheta}
\end{align}
with respect to the parameters $\theta$.\\
At each iteration of the stochastic gradient,
\begin{itemize}
 \item $M$ distributions  $(\mu^m)_{m=1,M}$ are generated and  for each $\mu^m$,  $X^{m,n} \sim \mu^m$ are generated for $n=1, \ldots, N$, following the methodology given in section \ref{sec:sampling}.
 \item The $K$ features representing the law  are estimated from the $N$ samples for a given  estimator $R^{K,N}(\mu^m) := (R^{K,N}_k( (X^{m,n})_{n=1,N}))_{k=1,\ldots,K}$.
 \end{itemize}
The discretized version of the loss function  \eqref{eq:probTheta} is then
\begin{align}
\tilde L(\theta)  :=  \; \frac{1}{M} \sum_{m=1}^M   \big| V(\mu^{m}) -  \Phi_\theta(R^{K,N}(\mu^m)) \big|^2.
\label{eq:probDis}
\end{align}
 The learning rate associated with the gradient method is  equal to $5 \times 10^{-3}$ in all the experiments.

\section{The networks}
\label{sec:theNets}
\subsection{The quantile network for one dimensional distribution}
\label{sec:quantnet}
Let   $\Dc_2(\R)$  be the subset of probability measures $\mu$ in $\Pc_2(\R)$ admitting a  density function $\mrp^\mu$  with respect to the Lebesgue measure $\lambda$ on $\R$. 
Fixing  $\Kc$ as  a bounded segment in $\R$, we want to approximate the functions of the distributions with support in $\Kc$.\\
For a distribution $\mu$, we note $F_{\mu}$ its cumulative distribution function and  we note $Q_\mu$ the quantile
function defined as $Q_\mu(p) = \inf\{x \in \Kc : p \le F_\mu(x) \}$. In the sequel, we also use the notation $Q_X$ for  $Q_\mu$ if $X \sim \mu$. \\
Choosing $K > 0 $, the main  characteristics of the distribution $\mu$ are given  by 
\begin{align}
\boQ_\mu^K = (Q_{\mu}(\frac{k}{K+1}))_{k=1,K}
\end{align}
which lies on $\Dc_K$ $:=$ $\{ Q := (q_k)_{k=1,K} : q_1 < \dots < q_K \}$. \\
A quantile network is thus an operator on $\Dc_2(\R)$ in the form
 \begin{align}
 \Nc_Q(\mu) =  \Phi_\theta(\boQ_\mu^K),
 \end{align}
 so setting $R^K(\mu)= \boQ_\mu^K$ in equation \eqref{eq:probTheta}.\\
 Let us denote by $\Dc_{C^1}(\Kc)$ the subset of elements $\mu$ in $\Dc_2(\R)$ with support in $\Kc $, with continuously derivable  density functions $\mrp^\mu$. We get the following universal approximation theorem:
 \begin{Theorem} \label{theounivbin} 
Let $\Kc =[ \underline \Kc, \bar \Kc] $ be a bounded segment in  $\R$, 
$V$  a continuous function from $\Pc_2(\R)$ to $\R$. Then, for all 
$\varepsilon$ $>$ $0$, there exists $K$ $\in$ $\N^*$, and 
$\Phi$ a neural network on $\R^K$ with values in $\R$ such that 
\begin{align}
\big| V(\mu) - \Phi(\boQ_\mu^K)\big| 
& \leq  \; \varepsilon, \quad \forall \mu \in \Dc_{C^1}(\Kc). 
\end{align}
\end{Theorem}
\begin{proof}
The proof is very similar to the proof of theorem 2.1 in \cite{pham2022mean}.
The only difference is the modification of step one in the proof.
From the quantile representation of the density function, we get the following density step approximation:
\begin{align*}
p^{\boQ}_{\mu,K}(x) =  \frac{1}{K (Q_{\mu,k+1}^K - Q_{\mu,k}^K )} , \mbox { for } x \in ]Q_{\mu,k}^K, Q_{\mu,k+1}^K],  \quad 0 \le k \le K \\
\end{align*}
where $Q_{\mu,k}^K = Q_\mu^K(\frac{k}{K+1})$, for $k = 1, \ldots, K$, $Q_{\mu,0}^K =\underline \Kc$ , $Q_{\mu,K+1}^K = \bar \Kc$.\\
For $\mu$ $\in$ $\Dc_{C^1}(\Kc)$  with density $\mrp^\mu$,  denote by $\hat \mu^K$ $=$ $\Lc_D(p^{\boQ}_{\mu,K})$  the probability measure with  density representation $p^{\boQ}_{\mu,K}$.  

Since $\mu$, $\hat \mu^K$ are supported on  the compact set $\Kc$, they lie in $\Pc_1(\R^d)$ the set of probability measures with finite first moment.
From the Kantorovich-Rubinstein dual representation of the $1$-Wasserstein distance, we have 
\beqs
\Wc_1(\mu,\hat\mu^K) & = & \sup_{\phi} \int_{\Kc}  \phi(x) (\mrp^\mu(x) - p^{\boQ}_{\mu,K}(x)) \d x,  
\enqs
where the supremum is taken over all Lipschitz continuous functions $\phi$ on $\Kc$ with Lipschitz constant bounded by $1$, and where we can assume w.l.o.g. that $\phi(x_0)$ $=$ $0$ for some fixed point $x_0$ in $\Kc$. \\
Noting $\bar \mrp^\mu_k := \frac{1}{(Q_{\mu,k+1}^K - Q_{\mu,k}^K ) }\int_{Q_{\mu,k}^K}^{Q_{\mu,k+1}^K} 
\mrp^\mu(s) ds  =  \mrp^\mu(\tilde x_k)$  with $\tilde  x_k \in [Q_{\mu,k}^K ,Q_{\mu,k+1}^K] $ due to mean value theorem:
\begin{align}
\Wc_1(\mu,\hat\mu^K) & \leq &  \sup_{\phi} \sum_{k=1}^K \int_{Q_{\mu,k}^K}^{Q_{\mu,k+1}^K} |\phi(x)| |\big( \mrp^\mu(x) - \frac{1}{K (Q_{\mu,k+1}^K - Q_{\mu,k}^K )}\big)| \d x \\
& \leq & {\rm diam}(\Kc) \sum_{k=1}^K \int_{Q_{\mu,k}^K}^{Q_{\mu,k+1}^K}  |\big( \mrp^\mu(x) - \frac{1}{K (Q_{\mu,k+1}^K - Q_{\mu,k}^K )}\big)| \d x \\
%& = & {\rm diam}(\Kc) \sum_{k=1}^K  \int_{Q_{\mu,k}^K}^{Q_{\mu,k+1}^K}   |\mrp^\mu(x) - \frac{1}{(Q_{\mu,k+1}^K - Q_{\mu,k}^K ) }\int_{Q_{\mu,k}^K}^{Q_{\mu,k+1}^K}  \mrp^\mu(s) ds| dx\\
& = &  {\rm diam}(\Kc) \sum_{k=1}^K  \int_{Q_{\mu,k}^K}^{Q_{\mu,k+1}^K}  |\mrp^\mu(x)  - \bar \mrp^\mu_k |  dx\\
& = &  {\rm diam}(\Kc) \sum_{k=1}^K  ( 1_{\bar \mrp^\mu_k < \epsilon} +  1_{\bar \mrp^\mu_k \ge \epsilon} ) \int_{Q_{\mu,k}^K}^{Q_{\mu,k+1}^K}  |\mrp^\mu(x)  - \bar \mrp^\mu_k |  dx 
\label{eq:p0}
\end{align}
where we used that  $|\phi(x)|$ $\leq$ $|x-x_0|$ $\leq$ ${\rm diam}(\Kc)$.\\
Then :
\begin{align}
 \sum_{k=1}^K  1_{\bar \mrp^\mu_k < \epsilon} \int_{Q_{\mu,k}^K}^{Q_{\mu,k+1}^K}  |\mrp^\mu(x)  - \bar \mrp^\mu_k |  dx \le & \sum_{k=1}^K  1_{\bar \mrp^\mu_k < \epsilon} \int_{Q_{\mu,k}^K}^{Q_{\mu,k+1}^K}  (\mrp^\mu(x)  + \bar \mrp^\mu_k ) dx  \\
& =    \sum_{k=1}^K  1_{\bar \mrp^\mu_k < \epsilon}  (Q_{\mu,k+1}^K -Q_{\mu,k}^K) 2  \bar \mrp^\mu_k \le  2 \epsilon
\label{eq:p1}
\end{align}
Notice that  if $\mrp^\mu_k \ge \epsilon$, $ Q_{\mu,k+1}^K -Q_{\mu,k}^K  \le \frac{1}{\epsilon K}  $ and again due to mean value theorem and noting $C= \sup_y \mrp^{\mu'}(y)$:
\begin{align}
\sum_{k=1}^K  1_{\bar \mrp^\mu_k  \ge \epsilon} \int_{Q_{\mu,k}^K}^{Q_{\mu,k+1}^K}  |\mrp^\mu(x)  - \bar \mrp^\mu_k |  dx  = & \sum_{k=1}^K  1_{\bar \mrp^\mu_k  \ge \epsilon} \int_{Q_{\mu,k}^K}^{Q_{\mu,k+1}^K}  |\mrp^\mu(x)  -  \mrp^\mu_k(\tilde x_k)| dx \\
\le &  C \sum_{k=1}^K  1_{\bar \mrp^\mu_k  \ge \epsilon}  \int_{Q_{\mu,k}^K}^{Q_{\mu,k+1}^K} |x- \tilde x_k| dx \\
\le &  C  \sum_{k=1}^K   1_{\bar \mrp^\mu_k  \ge \epsilon}   (Q_{\mu,k+1}^K -Q_{\mu,k}^K) ^2 \\
\le &  C  \sum_{k=1}^K   (Q_{\mu,k+1}^K -Q_{\mu,k}^K) \frac{1}{\epsilon K} = C \frac{1}{\epsilon K}
\label{eq:p2}
\end{align}
Pluging equations \eqref{eq:p1}, \eqref{eq:p2} in \eqref{eq:p0} gives:
\begin{align}
\Wc_1(\mu,\hat\mu^K) \leq {\rm diam}(\Kc) (2  \epsilon + C \frac{1}{\epsilon K})
\end{align}

\begin{comment}
Then still using the mean value theorem and noting $C= \sup_y \mrp^{\mu'}(y)$
\begin{align}
\Wc_1(\mu,\hat\mu^K) & \leq &  {\rm diam}(\Kagentsc) \sum_{k=1}^K  [ 2 \epsilon (Q_{\mu,k+1}^K -Q_{\mu,k}^K) + \int_{Q_{\mu,k}^K}^{Q_{\mu,k+1}^K} 1_{(\mrp^\mu(x) > \epsilon} C| x - \tilde x| dx ]\\
& \le & {\rm diam}(\Kc)^2  2  \epsilon +  C  \sum_{k=1}^K  (Q_{\mu,k+1}^K - Q_{\mu,k}^K)\int_{Q_{\mu,k}^K}^{Q_{\mu,k+1}^K}  1_{\mrp^\mu(x) > \epsilon} \frac{\mrp^\mu(x)}{\epsilon} dx \\
& \le &  {\rm diam}(\Kc)^2  2  \epsilon  +  C  \sum_{k=1}^K  (Q_{\mu,k+1}^K - Q_{\mu,k}^K) \frac{1}{\epsilon K} \\
& \le & {\rm diam}(\Kc)^2  2 \epsilon +  C {\rm diam}(\Kc) \frac{1}{K \epsilon} 
\end{align}
\end{comment}
For $\epsilon$ given, it is possible to have $K$ high enough to get $
\Wc_1(\mu,\hat\mu^K)  \le {\rm diam}(\Kc)  3 \epsilon $ and noting that $\Wc_2(\mu,\hat\mu^K)$ $\leq$ $\sqrt{{\rm diam}(\Kc)\Wc_1(\mu,\hat\mu^K)}$ by Hölder inequality, we get that $\Wc_2(\mu,\hat\mu^K) \le   {\rm diam}(\Kc) \sqrt{ 3 \epsilon}$.\\
Therefore we have shown that
\begin{align} \label{convWK} 
\sup_{\mu \in \Dc_{C^1}(\Kc)} \Wc_2(\mu,\hat \mu^K)  & \rightarrow \; 0, \quad \mbox{ as } K \rightarrow \infty. 
\end{align} 
Then we use the same argument as in \cite{pham2022mean} to get that, for a given $\epsilon$, it is possible to set $K$ such that
\begin{align} \label{inegV1} 
| V(\mu) - V(\hat \mu^K) | & \leq \; \frac{\varepsilon}{2}, \quad \forall   \; \mu \in \Dc_{C^1}(\Kc). 
\end{align}
At last using a classical universal theorem, we can conclude as in Step 2 of \cite{pham2022mean}.

\end{proof}

\subsection{The moment network}
Let  $\Dc_2(\R^d)$ be the subset of probability measures $\mu$ in $\Pc_2(\R^d)$ that admit a density function $\mrp^\mu$  with respect to the Lebesgue measure $\lambda_d$ on $\R^d$. 
Fixing  $\Kc$ as  a bounded rectangle in $\R^d$, we want to approximate the function of the distribution with support in $\Kc$.
By choosing $K > 0 $, the main  features of the distribution $\mu$ are approximated  by  choosing the lowest moments of the distribution so:
\begin{align}
\boM_\mu^{K} = (\E _{X \sim \mu} [ \prod_{i=1,\ldots,d} X_i^{k_i}])_{ \sum_{i=1}^d k_i \le  K}
\end{align}
with values in $\R^{\hat K}$, with $\hat K = \#\{ p \in \N^d / \sum_{i=1}^d p_i \le K \}$.
\begin{Remark}
Since the support of the distribution is bounded all moments are well defined.
\end{Remark}
A moment network is an operator on $\Dc_2(\R^d)$ in the form
 \begin{align}
 \Nc_Q(\mu) =  \Phi_\theta(\boM_\mu^K),
 \end{align}
 so setting $R^{\hat K}(\mu)=  \boM_\mu^K$ in the equation \eqref{eq:probTheta}.
 
 \begin{Remark}
 This approach is closely related to the moment problem which consists in determining a distribution from its moments if they exist. If the support of the distribution is $[0, \infty[$, this is the Stieltjes moment problem, and if the support is $\R$, this is the Hamburger moment problem.
If $\mu$ is a positive measure with all moments defined, we say that $\mu$  is a solution to the moment problem.
If the solution to the moment problem is unique, the moment problem is called determinate. Otherwise
the moment problem is said to be indeterminate.
In our case, where the support is compact, this problem is known as the Haussdorf moment problem and it is determinate.
The connection with the moment problem and the reconstruction of an approximation of the quantile has been studied for example in \cite{mnatsakanov2009recovery}.
\end{Remark}
We now give a universal approximation theorem for this neural network:
\begin{Theorem} \label{theounivmome} 
Let $\Kc$ be a bounded rectangle in  $\R^d$, and $V$ be a continuous function from $\Pc_2(\Kc)$ into $\R^p$,   then, for all $\varepsilon$ $>$ $0$, there exists 
$K$ and $\Psi$ a  neural network from $\R^{\hat K}$ to $\R^p$  such that
\begin{align}
 \big| V(\mu) - \Psi(\boM_\mu^K) \big|  
& \leq \; \varepsilon \quad \forall \mu \in \Pc(\Kc)
\end{align}
\end{Theorem}
\begin{proof}
By the density of the cylindrical polynomial function with respect to distribution functions, see Lemma 3.12 in \cite{guophawei22}, for all $\varepsilon$ $>$ $0$, there exists $K$ 
$\in$ $\N^*$,  $P$ a linear function from $\R^{\hat K}$ into $\R^p$, s.t.  
\begin{align}
\big| V(\mu) - P( \boM_\mu^K) \big| & \leq \;  \frac{\varepsilon}{2}, \quad \forall \mu \in \Pc(\Kc). 
\end{align}
Note that since $\Kc$  is  bounded, $\boM_\mu^K$ is in a compact $\Yc$ and  we  use  the  classical universal approximation theorem for finite-dimensional functions to obtain the existence of  a feedforward neural network $\Psi$ on $R^{\hat K}$ such that 
\begin{align}
\big| P(x) - \Psi(x) \big| & \leq \;  \frac{\varepsilon}{2}, \quad \forall  x \in \Yc. 
\end{align}
We conclude that for all  $\mu$ $\in$ $\Pc(\Kc)$, 
\begin{align} 
& \;  \big| V(\mu) - \Psi(\boM_\mu^K) \big| \nonumber \\
& \leq \;     \big| V(\mu) - P(\boM_\mu^K) \big|  +     \big| P(\boM_\mu^K) - \Psi(\boM_\mu^K) \big|  \; \leq \; \varepsilon.  \label{Vpsi} 
\end{align} 
\end{proof}

\section{Numerical tests}
\subsection{Univariate case}
All functions are tested with the developed networks, and the results are compared with those obtained using the
bin and cylinder methods in \cite{pham2022mean}.
In dimension one, we propose to learn the following functions $V$ with support in $[-2,2]$.
\begin{itemize}
    \item[A.]   The moment case
    \begin{align*}
        V(\mu)= \E_{X \sim \mu}[X] \E_{X \sim \mu}[X^4] - \E_{X \sim \mu}[X^2]
    \end{align*}
    \item[B.]   The pure  quantile case
    \begin{align*}
      V(\mu)=  Q_\mu(q)
    \end{align*}
    and we take $q=0.7$.
    \item[C.] The quantile-moment case
    \begin{align*}
      V(\mu)=   \E_{X \sim \mu}[X^3] (1+ Q_\mu(q))
    \end{align*}
    taking $q=0.9$.
        \item[D.]   The quantile-superquantile case
    \begin{align*}
      V(\mu)= \E_{X \sim \mu}[X / X > Q_\mu(q)] + Q_\mu(q)
    \end{align*}
    and we take $q=0.3$.
\end{itemize}
All distribution features are estimated with $N=200000$ samples, and the distribution is sampled using the method in the  section \ref{sec:sampling} using $J_1=400$ bins.
During the training, $M=20$ distributions (the batch size) are used. All curves plot the MSE obtained during gradient iterations  as follows : 
every 100 iterations, the MSE is estimated using $1000$ distributions and the results are plotted using a window averaging the estimates over $20$ consecutive iterations.
The ReLU activation function is used for all networks. Similar results are obtained using the $\tanh$ activation function.
The quantile, and moment networks  (and the networks derived from these features) use 2 hidden layers with $20$ neurons. The cylinder network, which uses 2 networks, has  $3$ layers and $20$ neurons for the "inner" network and $2$ layers with $20$ neurons for the "outer" network (see \cite{pham2022mean}).
\begin{figure}[H]
     \centering
     \begin{minipage}[t]{0.49\linewidth}
  \centering
 \includegraphics[width=\textwidth]{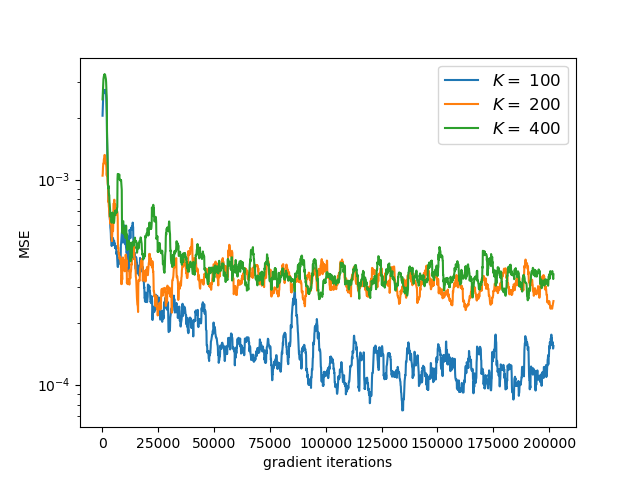}
\caption*{Case A}
\end{minipage}
 \begin{minipage}[t]{0.49\linewidth}
  \centering
 \includegraphics[width=\textwidth]{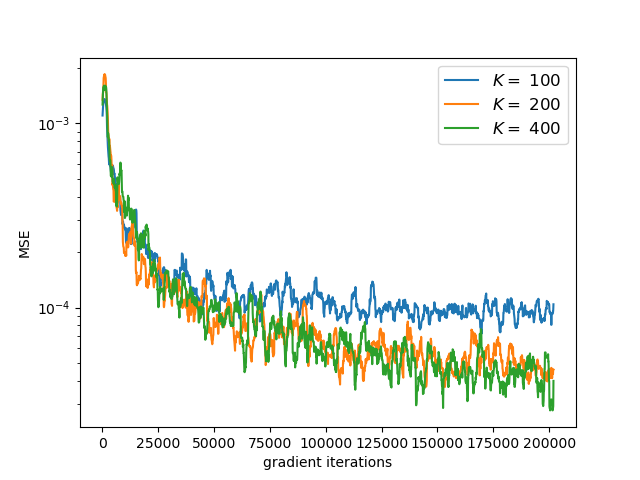}
\caption*{Case B}
\end{minipage}
   \centering
     \begin{minipage}[t]{0.49\linewidth}
  \centering
 \includegraphics[width=\textwidth]{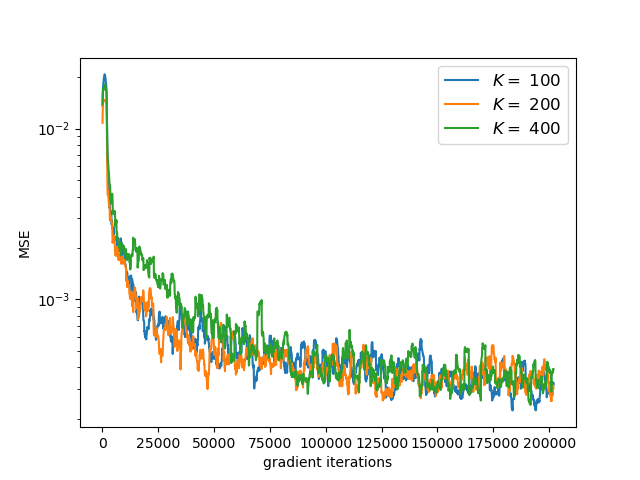}
\caption*{Case C}
\end{minipage}
     \begin{minipage}[t]{0.49\linewidth}
  \centering
 \includegraphics[width=\textwidth]{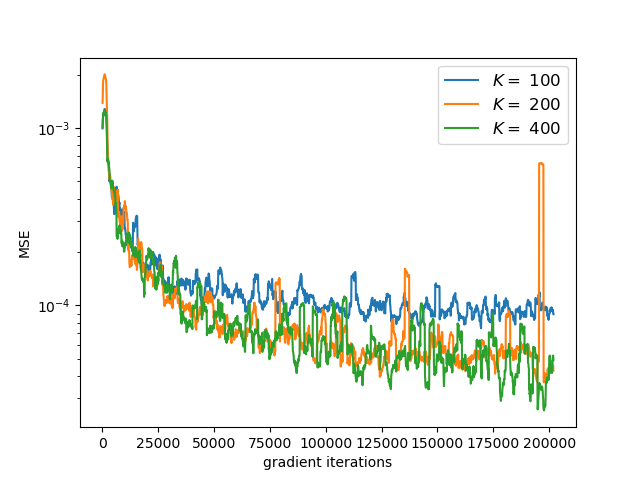}
\caption*{Case D}
\end{minipage}
     \caption{Quantile network convergence depending on the number of quantiles.}
     \label{fig:binQuant}
 \end{figure}
 The neural network seems to be less accurate for functions involving moments (cases A and C) (see figure \ref{fig:binQuant}).
 A number of quantiles equal to $200$ seems to be sufficient to obtain a good accuracy.
 
\begin{figure}[H]
     \centering
     \begin{minipage}[t]{0.49\linewidth}
  \centering
 \includegraphics[width=\textwidth]{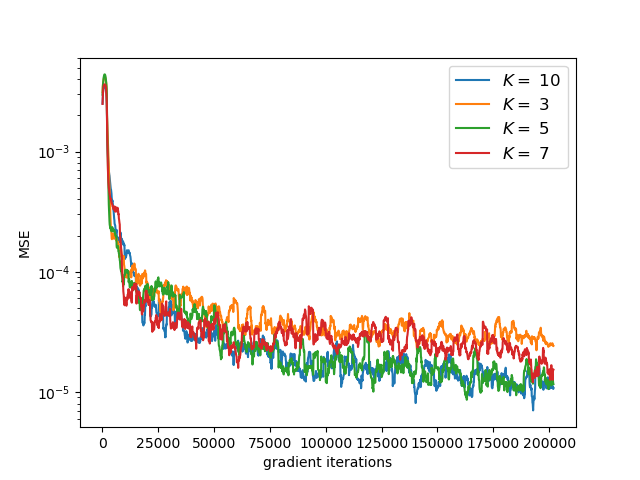}
\caption*{Case A}
\end{minipage}
 \begin{minipage}[t]{0.49\linewidth}
  \centering
 \includegraphics[width=\textwidth]{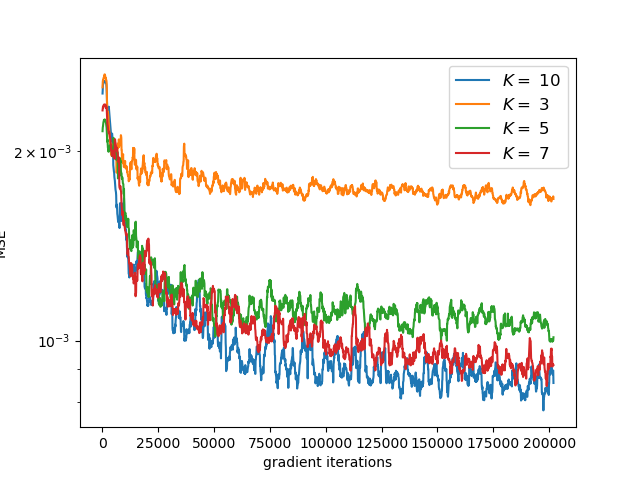}
\caption*{Case B}
\end{minipage}
   \centering
     \begin{minipage}[t]{0.49\linewidth}
  \centering
 \includegraphics[width=\textwidth]{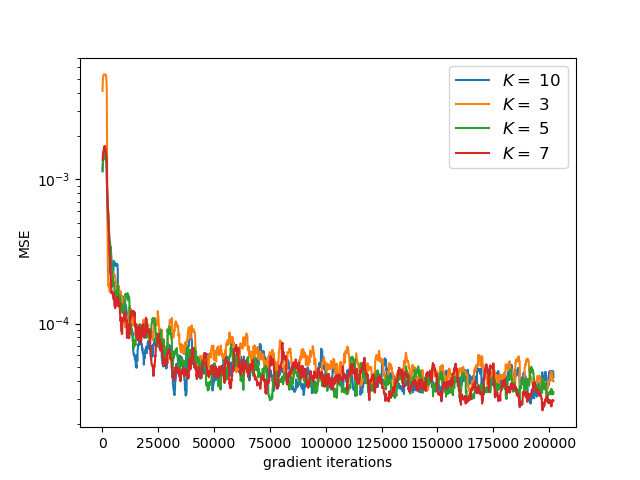}
\caption*{Case C}
\end{minipage}
   \centering
     \begin{minipage}[t]{0.49\linewidth}
  \centering
 \includegraphics[width=\textwidth]{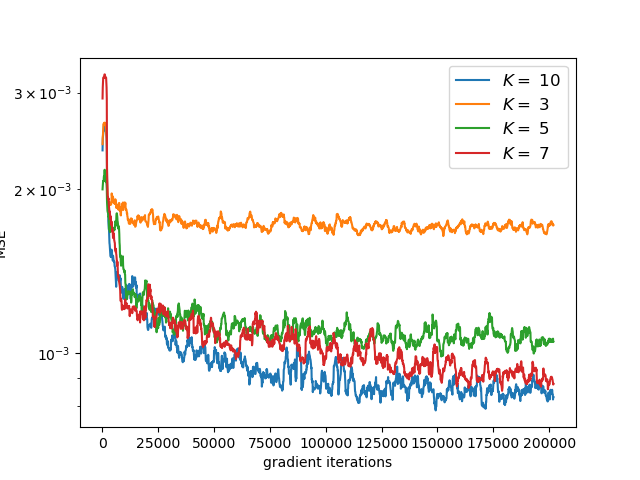}
\caption*{Case D}
\end{minipage}
     \caption{Moment network convergence depending on the number of moments.}
     \label{fig:Mom1D}
 \end{figure}
 In contrast to the quantile networks, the results are not surprisingly better when the functional to be approximated is mainly a function of the moments (see Figure \ref{fig:Mom1D}).
 In case A, it is optimal to use a small number of moments, since the functional is only a function of moments with degrees less than 5.\\
 Since the best network depends on the case, we can develop  new networks based on moments and quantiles:
 \begin{itemize}
 \item A first one uses a concatenation of the features of the two proposed networks.
 Using the same notation as in the section \ref{sec:theNets}, 
 \begin{align}
 \Nc_{QM}(\mu) =  \Phi_\theta(\boM_\mu^{K^M}, \boQ_\mu^{K^Q}),
 \end{align}
 where now $K^M$ is the number of moments retained in the approximation and $K^Q$ is the number of quantiles.
 This neural network is the {\bf moment and quantile network}. The results obtained for this network are shown in Figure
\ref{fig:QuantAndMom}.   Overall, it seems that a moment number of $K^M=7$ and a quantile number of $K^Q=200$ is a good choice.
 \begin{figure}[H]
     \centering
     \begin{minipage}[t]{0.49\linewidth}
  \centering
 \includegraphics[width=\textwidth]{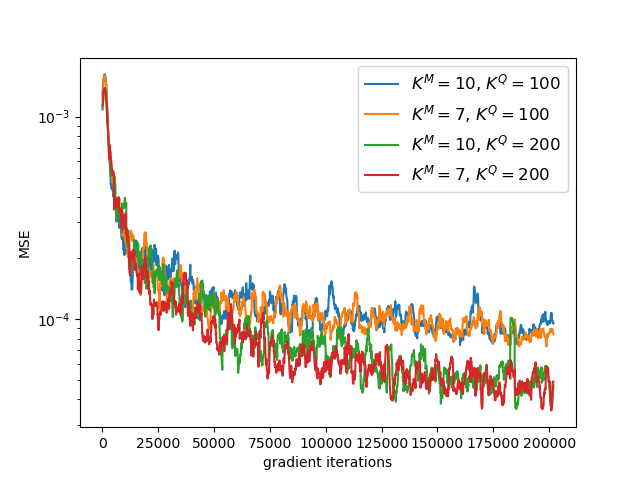}
\caption*{Case A}
\end{minipage}
 \begin{minipage}[t]{0.49\linewidth}
  \centering
 \includegraphics[width=\textwidth]{1DDistOnlyNetConvActiv1DNetMomQuantNetFuncEQDNbsim200000NbL20NbNe2.png}
\caption*{Case B}
\end{minipage}
   \centering
     \begin{minipage}[t]{0.49\linewidth}
  \centering
 \includegraphics[width=\textwidth]{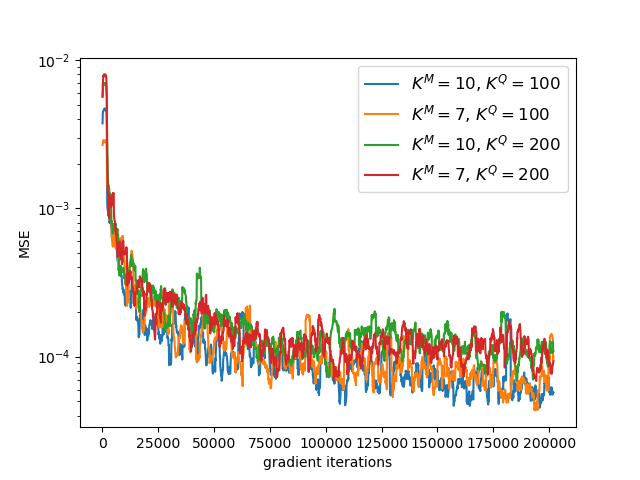}
\caption*{Case C}
\end{minipage}
 \centering
     \begin{minipage}[t]{0.49\linewidth}
  \centering
 \includegraphics[width=\textwidth]{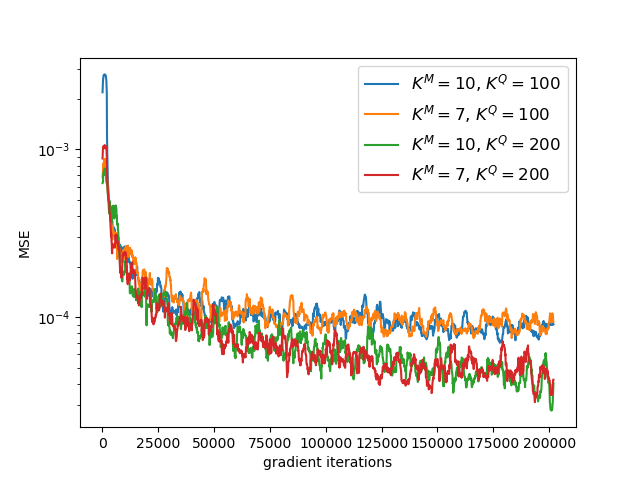}
\caption*{Case D}
\end{minipage}
     \caption{Moment and quantile network convergence depending on $K^M$ and $K^Q$.}
     \label{fig:QuantAndMom}
 \end{figure}

\item In a second one, instead of concatenating some moments expectations  and quantiles,  we can take some quantiles of the moments by defining:

 \begin{align}
\boL_\mu^{K_M, K_Q} = \big[ Q_{\prod_{i=1,\ldots,d} X_i^{ \tilde k_i} / X \sim \mu} ( \frac{ \hat k}{K_Q+1} ) \big] _{ \sum_{i=1}^d \tilde k_i \le K_M,  1 \le \hat k \le K_Q }
\end{align}
A {\bf quantile of moments network} is an operator on $\Dc_2(\R^d)$ in the form
 \begin{align}
 \Nc_Q(\mu) =  \Phi_\theta(\boL_\mu^{K_M, K_Q} ),
 \end{align}
 thus setting $R^K(\mu)=  \boL_\mu^{K_M, K_Q}$ in the equation \eqref{eq:probTheta}. The results for this network are shown in the figure \ref{fig:QuantOfMom}. The  convergence seems to be good in all cases, but we observe that this convergence is less regular than with the previous neural network.

 \begin{figure}[H]
     \centering
     \begin{minipage}[t]{0.49\linewidth}
  \centering
 \includegraphics[width=\textwidth]{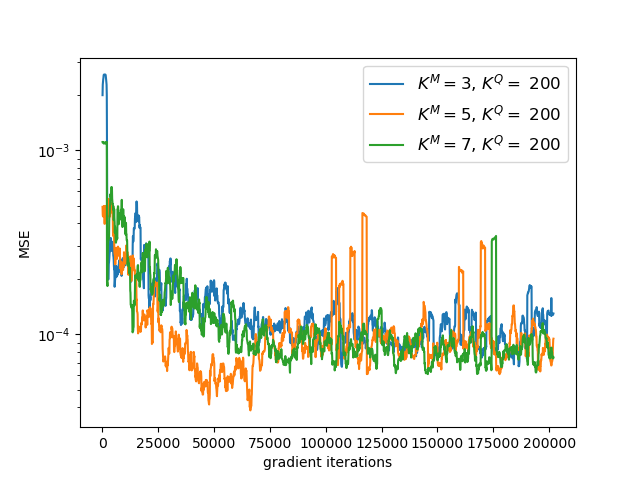}
\caption*{Case A}
\end{minipage}
 \begin{minipage}[t]{0.49\linewidth}
  \centering
 \includegraphics[width=\textwidth]{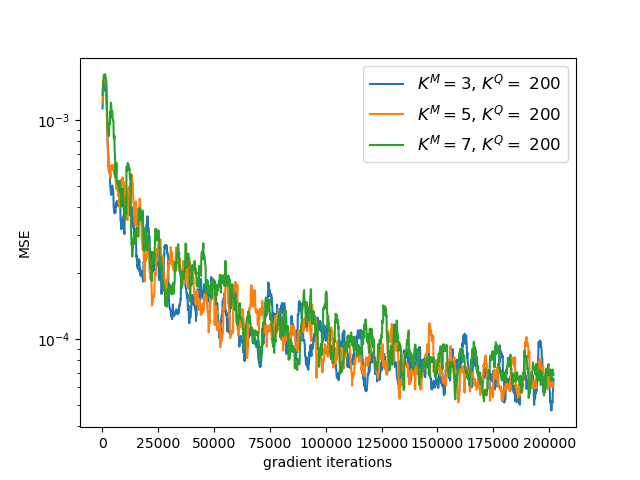}
\caption*{Case B}
\end{minipage}
   \centering
     \begin{minipage}[t]{0.49\linewidth}
  \centering
 \includegraphics[width=\textwidth]{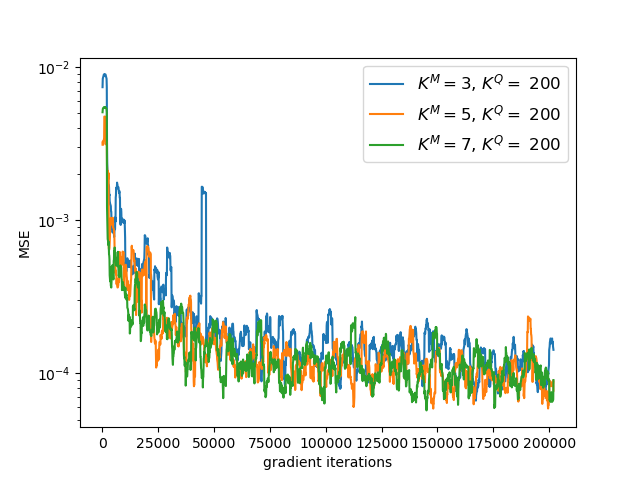}
\caption*{Case C}
\end{minipage}
 \centering
     \begin{minipage}[t]{0.49\linewidth}
  \centering
 \includegraphics[width=\textwidth]{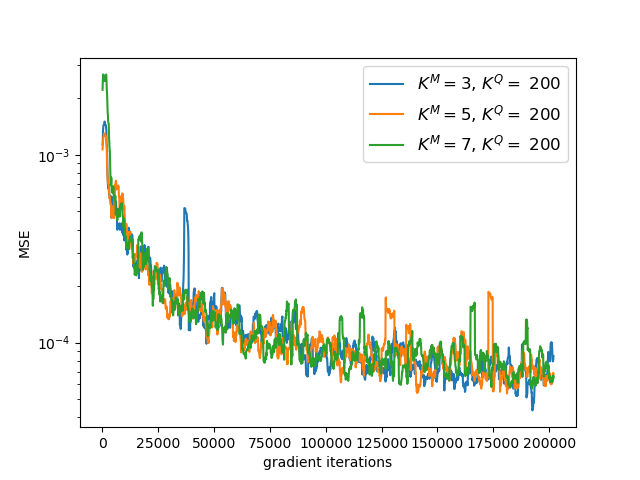}
\caption*{Case D}
\end{minipage}
     \caption{Quantile of moments network convergence.}
     \label{fig:QuantOfMom}
 \end{figure}
 \end{itemize}
 \begin{Remark}
 We have  also tested  networks based on superquantiles or superquantiles of moments.
 \begin{itemize}
 \item Defining for one-dimensional distributions
 \begin{align}
\boV_\mu^{ K} = \big[ E_{X \sim \mu}[  X  \ge  Q_{\mu} ( \frac{ \hat k}{K+1} ) ]\big] _{ 0 \le \hat k \le K},
\end{align}
a superquantile  network is an operator on $\Dc_2(\R^d)$ in the form
 \begin{align}
 \Nc_Q(\mu) =  \Phi_\theta(\boV_\mu^{ K} ),
 \end{align}
 \item and defining for potentially multivariate distributions:
 \begin{align}
\boS_\mu^{K_M, K_Q} = \big[ E_{X \sim \mu}[  \prod_{i=1,\ldots,d} X_i^{ \tilde k_i} / \prod_{i=1,\ldots,d} X_i^{ \tilde k_i} \ge  Q_{\prod_{i=1,\ldots,d} X_i^{ \tilde k_i} / X \sim \mu} ( \frac{ \hat k}{K_Q+1} ) ]\big] _{ \sum_{i=1}^d \tilde k_i \le K_M,  0 \le \hat k \le K_Q },
\end{align}
a superquantile of moment network is an operator on $\Dc_2(\R^d)$ in the form
 \begin{align}
 \Nc_Q(\mu) =  \Phi_\theta(\boS_\mu^{K_M, K_Q} ).
 \end{align}
\end{itemize}
Both neural networks give good results but never better than the cylinder network. We don't report them.
\end{Remark}

We now compare all the networks together on the figure \ref{fig:Comp1D} using 
\begin{itemize}
\item $K=200$ quantiles for the quantile network, 
\item $K=10$ moments for the moment network, 
\item $K^M=7$ moments and $K^Q=200$ quantiles for the "quantile of moments" and "moment and quantile"  networks, 
\item $200$ bins for the bin network.
\end{itemize}
Overall, the "moment and quantile" network and the "quantile of moments" network seem to be the best choices.
 
 \begin{figure}[H]
     \centering
     \begin{minipage}[t]{0.49\linewidth}
  \centering
 \includegraphics[width=\textwidth]{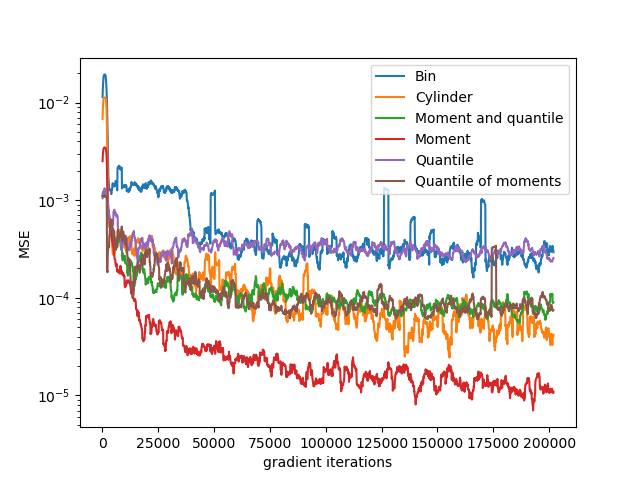}
\caption*{Case A}
\end{minipage}
 \begin{minipage}[t]{0.49\linewidth}
  \centering
 \includegraphics[width=\textwidth]{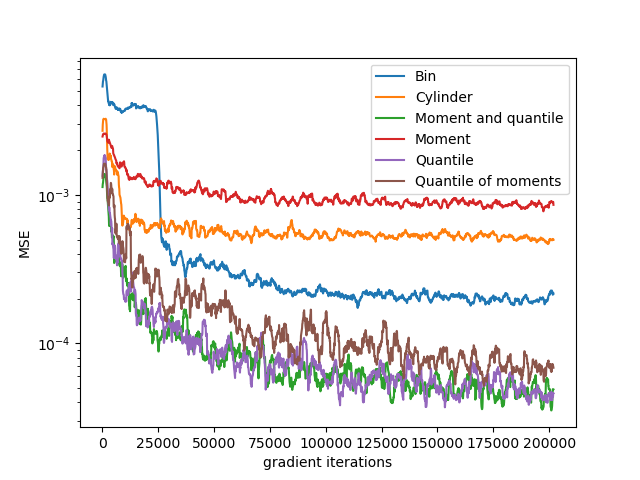}
\caption*{Case B}
\end{minipage}
   \centering
     \begin{minipage}[t]{0.49\linewidth}
  \centering
 \includegraphics[width=\textwidth]{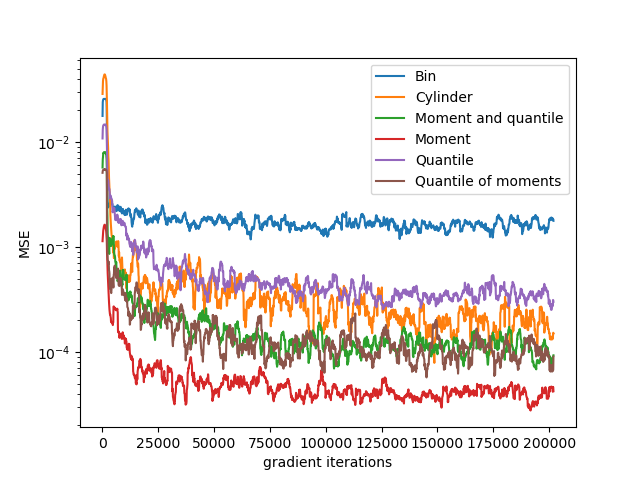}
\caption*{Case C}
\end{minipage}
\centering
     \begin{minipage}[t]{0.49\linewidth}
  \centering
 \includegraphics[width=\textwidth]{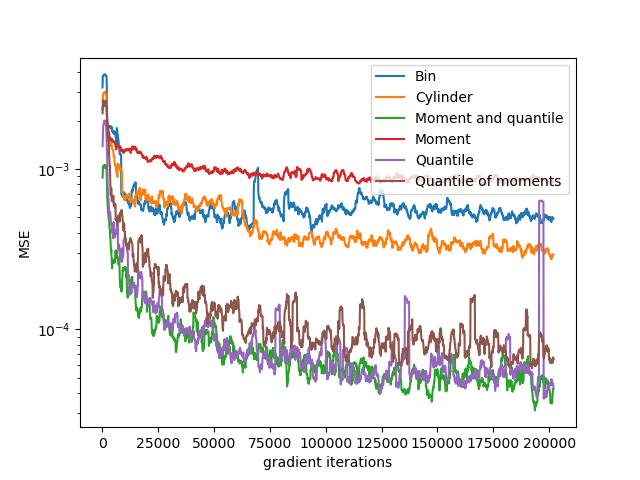}
\caption*{Case D}
\end{minipage}
     \caption{Convergence of the different networks in 1D.}
     \label{fig:Comp1D}
 \end{figure}
Finally, we compare the different networks, taking for all networks 3 layers and 40 neurons.
 \begin{figure}[H]
     \centering
     \begin{minipage}[t]{0.49\linewidth}
  \centering
 \includegraphics[width=\textwidth]{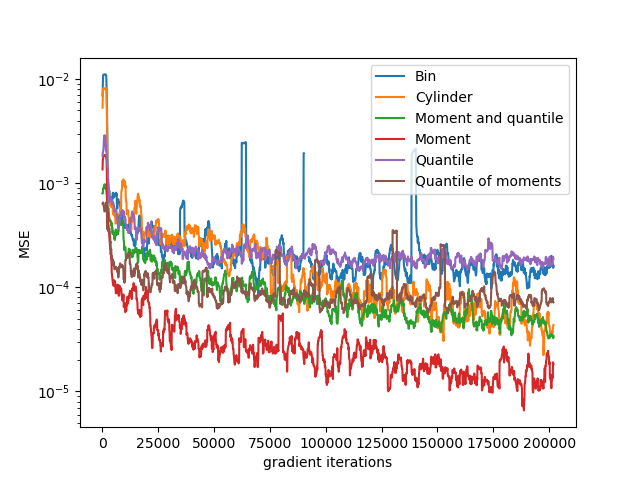}
\caption*{Case A}
\end{minipage}
 \begin{minipage}[t]{0.49\linewidth}
  \centering
 \includegraphics[width=\textwidth]{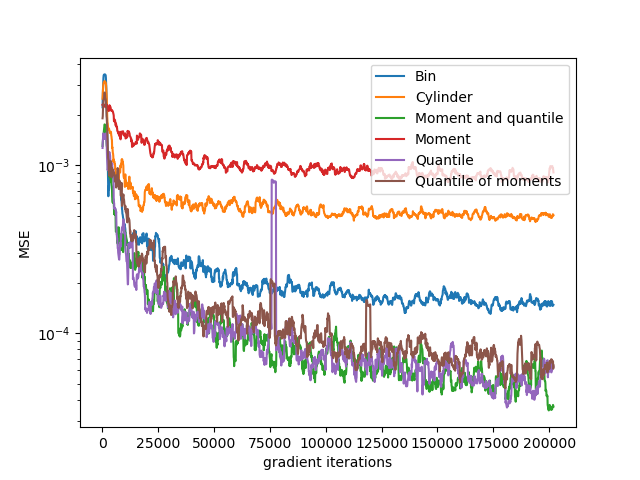}
\caption*{Case B}
\end{minipage}
   \centering
     \begin{minipage}[t]{0.49\linewidth}
  \centering
 \includegraphics[width=\textwidth]{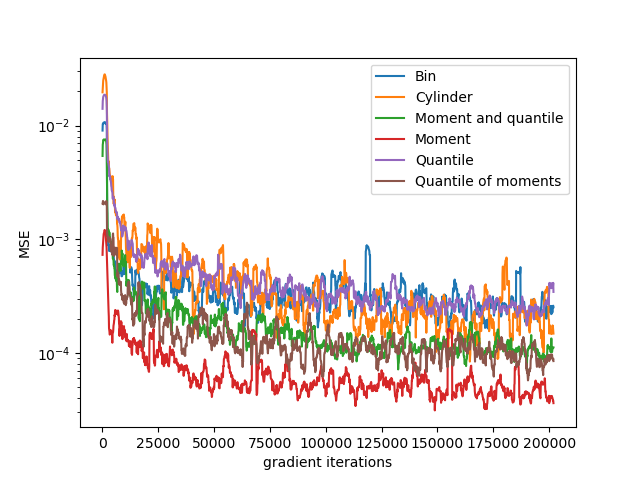}
\caption*{Case C}
\end{minipage}
\centering
     \begin{minipage}[t]{0.49\linewidth}
  \centering
 \includegraphics[width=\textwidth]{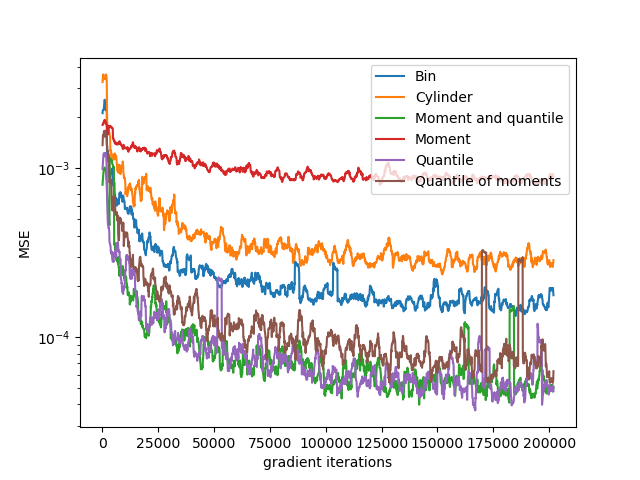}
\caption*{Case D}
\end{minipage}
     \caption{Convergence of the different networks in 1D taking 3 layers and 40 neurons.}
     \label{fig:Comp1DRaf}
 \end{figure}
  The results for the bin network are improved, but the conclusions remain the same.
  
\subsection{Bivariate case}
We assume that the support is in $[-2,2]^2$.
For a distribution $\mu \in \Pc_2(\R^d)$, $(j,m) \in \N^* \times \N^*$, we note $\hat F_{\mu,j,m}$ the cumulative distribution function of $X_1^j X_2^m$ where $X \sim \mu$ and  $\hat Q_{\mu,j,m}(p)= \inf\{ x \in \R : p \le \hat F_{\mu,j,m}(x) \}$. We define the following test cases:
\begin{itemize}
\item[A.] The moment case
\begin{align*}
        V(\mu)=  \sum_{i=1}^2  \left[ \E_{X \sim \mu_i}[X] \E_{X \sim \mu_i}[X^4] - \E_{X \sim \mu_i}[X^2] \right].
    \end{align*}
\item [B.] The  quantile-superquantile case
\begin{align*}
V(\mu)= & \sum_{i=1}^2 [ \E_{X \sim \mu_i}[X / X > Q_{\mu_i}(q)] + Q_{\mu_i}(q)] + \\
& \quad   \E_{X \sim \mu}[X_1 X_2 / X_1 X_2  > \hat Q_{\mu, 1,1}(q)]
\end{align*}
with $q=0.7$.
\item[C.] The quantile moment case
\begin{align*}
V(\mu)= & (1+Q_{\mu_1}(q)) \E_{X \sim \mu_1}[X^3]  + \E_{X \sim \mu_2}[X^3] + \\
& \hat Q_{\mu, 2,1}(q) + \hat Q_{\mu, 1, 2}(q)
\end{align*}
with $q=0.9$.
\item[D.] The quantile-superquantile marginal case
\begin{align*}
V(\mu) = & \sum_{i=1}^2 [ \E_{X \sim \mu_i}[X / X > Q_{\mu_i}(q_i)] + Q_{\mu_i}(q_i)]
\end{align*}
with $q=(0.6,0.3)$.
\item[E.] The quantile-cross-superquantile case
\begin{align*}
  V(\mu) = &  \E_{X \sim \mu}[X_2 / X_2 > Q_{\mu_1}(q)]  + Q_{\mu_1}(q)
\end{align*}
with $q=0.2$.
\item [F.] The quantile marginal case
\begin{align*}
  V(\mu) = &   Q_{\mu_1}(q) +  Q_{\mu_2}(q)
\end{align*}
with $q=0.8$.
\end{itemize}
We test the bin network, the cylinder network, the moment network, and the quantile of moments network on the different cases.
The bin network fails in all the test cases with a number of layers equal to 2 or 3 and a number of neurons taken equal to 20, 40 and 80. As for the other networks, we keep the same number of layers and neurons as in the previous section.
For the moment network we use $K=7$, while for the quantile of moment network we use 
$K^M=5$ and $K^Q=200$.
The distribution features are estimated using $N=400000$ samples, and we take $(J_1,J_2)= (200,200)$ to sample a given distribution.
\begin{figure}[H]
     \centering
     \begin{minipage}[t]{0.49\linewidth}
  \centering
 \includegraphics[width=\textwidth]{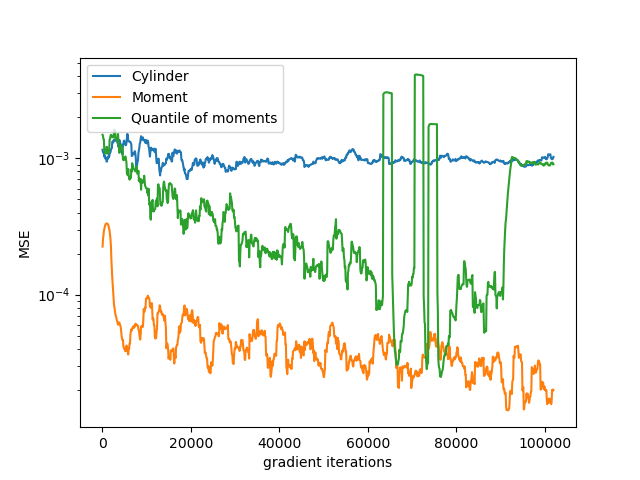}
\caption*{Case A}
\end{minipage}
 \begin{minipage}[t]{0.49\linewidth}
  \centering
 \includegraphics[width=\textwidth]{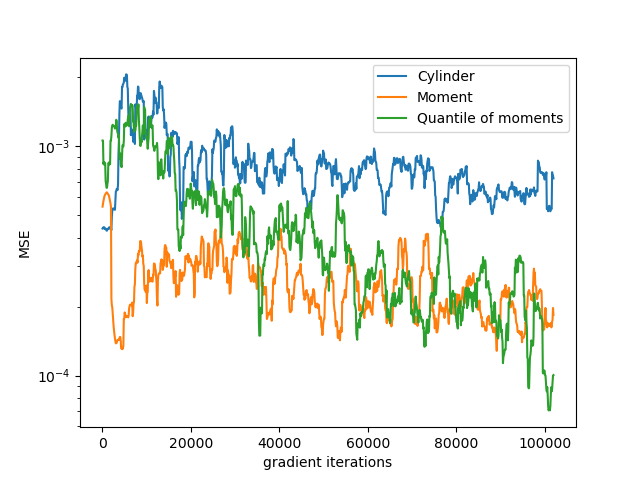}
\caption*{Case B}
\end{minipage}
   \centering
     \begin{minipage}[t]{0.49\linewidth}
  \centering
 \includegraphics[width=\textwidth]{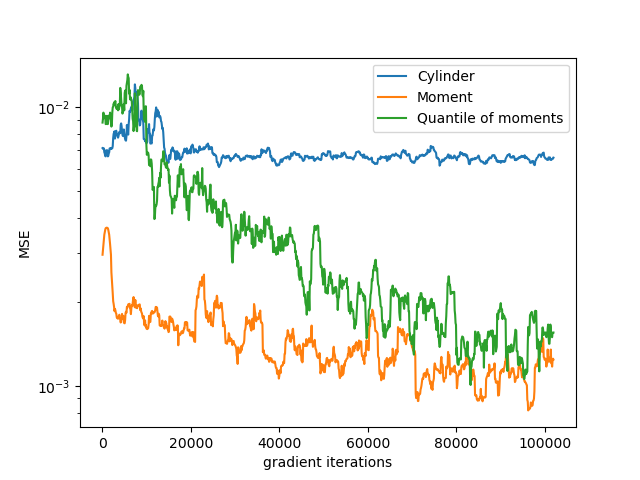}
\caption*{Case C}
\end{minipage}
   \centering
     \begin{minipage}[t]{0.49\linewidth}
  \centering
 \includegraphics[width=\textwidth]{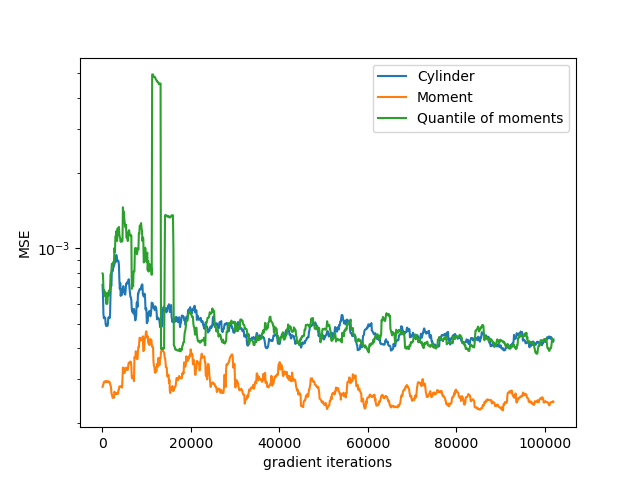}
\caption*{Case D}
\end{minipage}
   \centering
     \begin{minipage}[t]{0.49\linewidth}
  \centering
 \includegraphics[width=\textwidth]{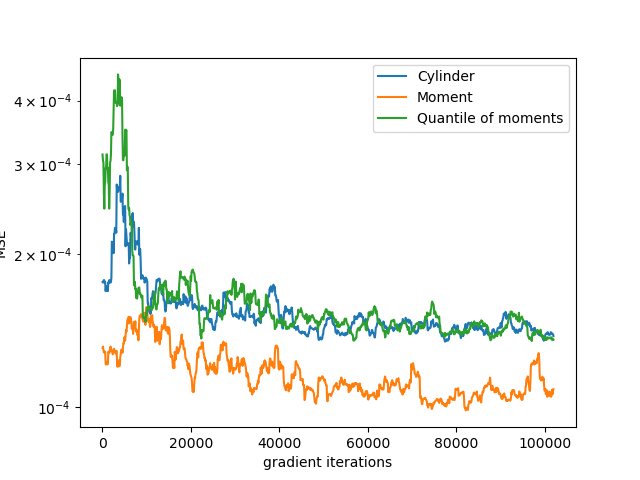}
\caption*{Case E}  \centering
\end{minipage}
    \begin{minipage}[t]{0.49\linewidth}
  \centering
 \includegraphics[width=\textwidth]{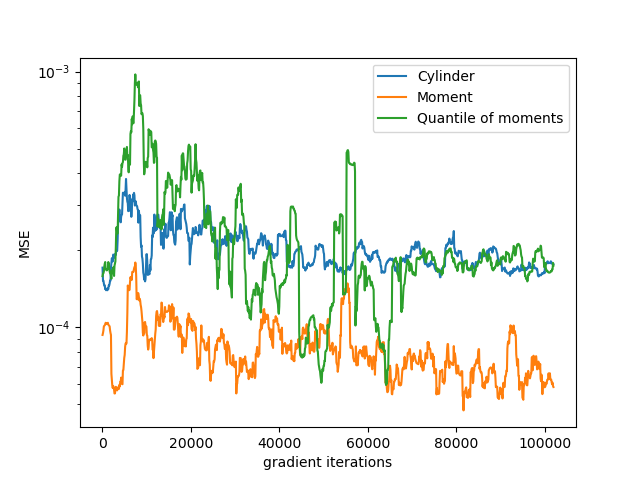}
\caption*{Case F}
\end{minipage}
     \caption{Convergence of the different networks in 2D.}
     \label{fig:Comp2D}
 \end{figure}
 The tests are shown in figure \ref{fig:Comp2D}.
 In all cases, the moment network gives the best results. We observe a loss not as good as in dimension one for the case C.

\section{Conclusion}
New networks have been developed to learn functions of distributions: some of them outperform the existing ones.
In all cases, the best networks are based on some moments of the distribution.
For univariate distributions, it is optimal to add some information, for example based on quantiles, to get an effective scheme on all the test cases.
For bivariate distributions, it is sufficient to take the expectation of the moments to get the best scheme. 
Using this moment scheme, the resolution of the PDE in Wasserstein space becomes possible in the multivariate case.

\printbibliography

@article{gerlauphawar21a,
	author = {M. Germain and M. Lauri{\`e}re and H. Pham and X. Warin},
	journal = {Journal of Scientific Computing},
	title = {DeepSets and derivative networks for solving symmetric {PDE}s},
	volume = {91},
	year = {2022}}

@article{pham2022mean,
  title={Mean-field neural networks: learning mappings on Wasserstein space},
  author={Pham, Huy{\^e}n and Warin, Xavier},
  journal={arXiv preprint arXiv:2210.15179},
  year={2022}
}

@article{kingma2014adam,
  title={Adam: A method for stochastic optimization},
  author={Kingma, Diederik P and Ba, Jimmy},
  journal={arXiv preprint arXiv:1412.6980},
  year={2014}
}

@inproceedings{abadi2016tensorflow,
  title={Tensorflow: a system for large-scale machine learning.},
  author={Abadi, Mart{\'\i}n and Barham, Paul and Chen, Jianmin and Chen, Zhifeng and Davis, Andy and Dean, Jeffrey and Devin, Matthieu and Ghemawat, Sanjay and Irving, Geoffrey and Isard, Michael and others},
  booktitle={Osdi},
  volume={16},
  number={2016},
  pages={265--283},
  year={2016},
  organization={Savannah, GA, USA}
}

@article{mnatsakanov2009recovery,
  title={Recovery of distributions via moments},
  author={Mnatsakanov, Robert M and Hakobyan, Artak S},
  journal={Lecture Notes-Monograph Series},
  pages={252--265},
  year={2009},
  publisher={JSTOR}
}

@article{guophawei22,
	author = {Guo, X. and Pham, H. and Wei, X.},
	journal = {arXiv: 2010.05288},
	title = {It\^o's formula for flows of semimartingales},
	year = {2022}}

@article{beck2020overview,
  title={An overview on deep learning-based approximation methods for partial differential equations},
  author={Beck, Christian and Hutzenthaler, Martin and Jentzen, Arnulf and Kuckuck, Benno},
  journal={arXiv preprint arXiv:2012.12348},
  year={2020}
}

@incollection{GPW21,
	author = {Germain, M. and Pham, H. and Warin, X.},
	booktitle = {arXiv:2101.08068 to appear in Machine Learning And Data Sciences For Financial Markets: A Guide To Contemporary Practices},
	editor = {Capponi, A. and Lehalle, C.A.},
	publisher = {Cambridge University Press},
	title = {Neural networks based algorithms for stochastic control and {PDE}s in finance},
	year = {2022}}

@book{cardel19,
	author = {Carmona, R. and Delarue, F.},
	publisher = {Springer},
	title = {Probabilistic Theory of Mean Field Games: vol. I, Mean Field FBSDEs, Control, and Games,},
	year = {2018}}

@book{cardel2,
	author = {Carmona, R. and Delarue, F.},
	publisher = {Springer},
	title = {Probabilistic Theory of Mean Field Games: vol. II, Mean Field FBSDEs, Control, and Games,},
	year = {2018}}

@article{warinpham2022meanW,
  title={Mean-field neural networks-based algorithms for McKean-Vlasov control problems},
  author={Pham, Huy{\^e}n and Warin, Xavier},
  journal={arXiv preprint arXiv:2212.11518},
  year={2022}
}

\end{document}